%% file: ppffoldnet.tex
\documentclass[runningheads]{llncs}
\usepackage{graphicx}
\usepackage{amsmath,amssymb,mathtools}
\usepackage{color}
\usepackage[utf8]{inputenc}
\usepackage{color}
\usepackage{times}
\usepackage{enumerate}
\usepackage{subfigure}
\usepackage{algorithm} 
\usepackage{algpseudocode}
\usepackage{booktabs}
\usepackage{tabularx}
\usepackage{array}
\usepackage{bm}
\usepackage{url}
\usepackage{bbm}
\usepackage{enumitem}
\usepackage{wrapfig}
\usepackage{appendix}

\makeatletter
\newcommand{\printfnsymbol}[1]{%
  \textsuperscript{\@fnsymbol{#1}}%
}
\makeatother
\input{commands.tex}

\begin{document}

\title{PPF-FoldNet: Unsupervised Learning of Rotation Invariant 3D Local Descriptors}
\titlerunning{PPF-FoldNet}
\authorrunning{Haowen Deng, Tolga Birdal, Slobodan Ilic}

\author{Haowen Deng \textsuperscript{\cmark\,\zmark\,\xmark} \qquad  Tolga Birdal \textsuperscript{\cmark\,\zmark} \qquad Slobodan Ilic \textsuperscript{\cmark\,\zmark}}

\institute{\textsuperscript{\cmark}\,\,\,Technische Universit\"{a}t M\"{u}nchen, Germany \quad
\textsuperscript{\zmark}\,\,\,Siemens AG, Germany\\
\textsuperscript{\xmark}\,\,\,National University of Defense Technology, China
}

\maketitle

\begin{abstract}
We present PPF-FoldNet for unsupervised learning of 3D local descriptors on pure point cloud geometry. Based on the folding-based auto-encoding of well known point pair features, PPF-FoldNet offers many desirable properties: it necessitates neither supervision, nor a sensitive local reference frame, benefits from point-set sparsity, is end-to-end, fast, and can extract powerful rotation invariant descriptors. Thanks to a novel feature visualization, its evolution can be monitored to provide interpretable insights. Our extensive experiments demonstrate that despite having six degree-of-freedom invariance and lack of training labels, our network achieves state of the art results in standard benchmark datasets and outperforms its competitors when rotations and varying point densities are present. PPF-FoldNet achieves $9\%$ higher recall on standard benchmarks, $23\%$ higher recall when rotations are introduced into the same datasets and finally, a margin of $>35\%$ is attained when point density is significantly decreased.
\end{abstract}

\input{introduction.tex}
\input{related.tex}

\input{method.tex}
\input{results.tex}
\input{conclusion.tex}
\clearpage

\bibliographystyle{splncs04}

\input{ppffoldnet.bbl}
\setcounter{section}{0}
\renewcommand\thesection{\Alph{section}}
\newcommand{\suppsection}{\subsection}
\clearpage
\section{Appendix}
\input{appendix.tex}

\end{document}

%% file: commands.tex
\usepackage{pifont}
\usepackage[symbol*]{footmisc}

\DefineFNsymbolsTM{otherfnsymbols}{%
  \textbullet \circ
  \textsection   \mathsection
  \textdagger    \dagger
  \textdaggerdbl \ddagger
  \textasteriskcentered *
  \textbardbl    \|%
  \textparagraph \mathparagraph
}%

\setfnsymbol{otherfnsymbols}
\newcommand{\cmark}{\ding{61}}%
\newcommand{\xmark}{*}%
\newcommand{\zmark}{\ding{67}}%

\renewcommand{\vec}[1]{\mathbf{#1}}

\newcommand{\insertimageStar}[5]{ 
\begin{figure*}[#5]
\centering
\includegraphics[width=#1\linewidth, clip=true]{figures/#2}
\caption{#3}
\label{#4}
\end{figure*}
}

\algnewcommand\algorithmicinput{\textbf{Input:}}
\algnewcommand\INPUT{\item[\algorithmicinput]}
\algnewcommand\algorithmicoutput{\textbf{Output:}}
\algnewcommand\OUTPUT{\item[\algorithmicinput]}

\newcommand{\suchthat}{\;\ifnum\currentgrouptype=16 \middle\fi|\;}

%% file: introduction.tex
\section{Introduction}
Local descriptors are one of the essential tools used in computer vision, easing the tasks of object detection, pose estimation, SLAM or image retrieval~\cite{lowe1999object,mur2015orb}. While being well established in the $2$D domain, $3$D local features are still known to lack good discriminative power and repeatibility. With the advent of deep learning, many areas in computer vision shifted from hand crafted labor towards a problem specific end-to-end learning. Local features are of course no exception. Already in $2$D, learned descriptors significantly outperform their engineered counterparts~\cite{yi2016lift,Noh_2017_ICCV}. Thus, it was only natural for the scholars to tackle the task of $3$D local feature extraction employing similar approaches~\cite{Khoury_2017_ICCV,zeng20163dmatch,ppfnet}. However, due to the inherent ambiguities and less informative nature of sole geometry, extracting $3$D descriptors on point sets still poses an unsolved problem, even for learning-based methods.

Up until now, deep learning of local features in $3$D has suffered from one or more of the following: \textbf{a}) being supervised and requiring an abundant amount of labels in form of pairs, triplets or $N$-tuples~\cite{zeng20163dmatch,ppfnet}, \textbf{b}) being sensitive to $6$DoF rotations~\cite{zeng20163dmatch,ppfnet}, \textbf{c}) involving significant hand-crafted input preparation~\cite{Khoury_2017_ICCV} and \textbf{d}) unsatisfactory performance~\cite{Khoury_2017_ICCV,qi2016pointnet}. In this paper, we map out an elegant architecture to tackle all of these problems and present PPF-FoldNet: an unsupervised, high-accuracy, $6$DoF transformation invariant, sparse and fast $3$D local feature learning network. PPF-FoldNet operates directly on point sets, taking into account the point sparsity and permutation invariant set property, deals well with density variations, while significantly outperforming its rotation-variant counterparts  even based on the standard benchmarks. 

Our network establishes theoretical rotation invariance inspired by use a point pair feature (PPF)~\cite{birdal2017cad,birdal3dv2015,ppfnet} encoding of the local $3$D geometry into patches. In contrast to PPFNet \cite{ppfnet}, we do not incorporate the original points or normals into the encoding. The collection of these $4$D PPFs are then sent to a FoldingNet-like end to end auto-encoder (AE)~\cite{foldingnet}, trained to auto-reconstruct the PPFs, using a set distance. Our encoder is simpler than in FoldingNet and for decoding, we propose a similar folding scheme, where a low dimensional $2$D grid lattice is folded onto a $4$D PPF space and monitor the network evolution by a novel lossless visualization of the PPF space. Our overall architecture is based on PointNet~\cite{qi2016pointnet} to achieve permutation invariance and to fully utilize the sparsity. Training our AE is far easier than training, for example, $3$DMatch~\cite{zeng20163dmatch}, because we do not need to sample pairs or triplets from a pre-annotated large dataset and we benefit from linear time complexity to the number of patches.
\insertimageStar{1}{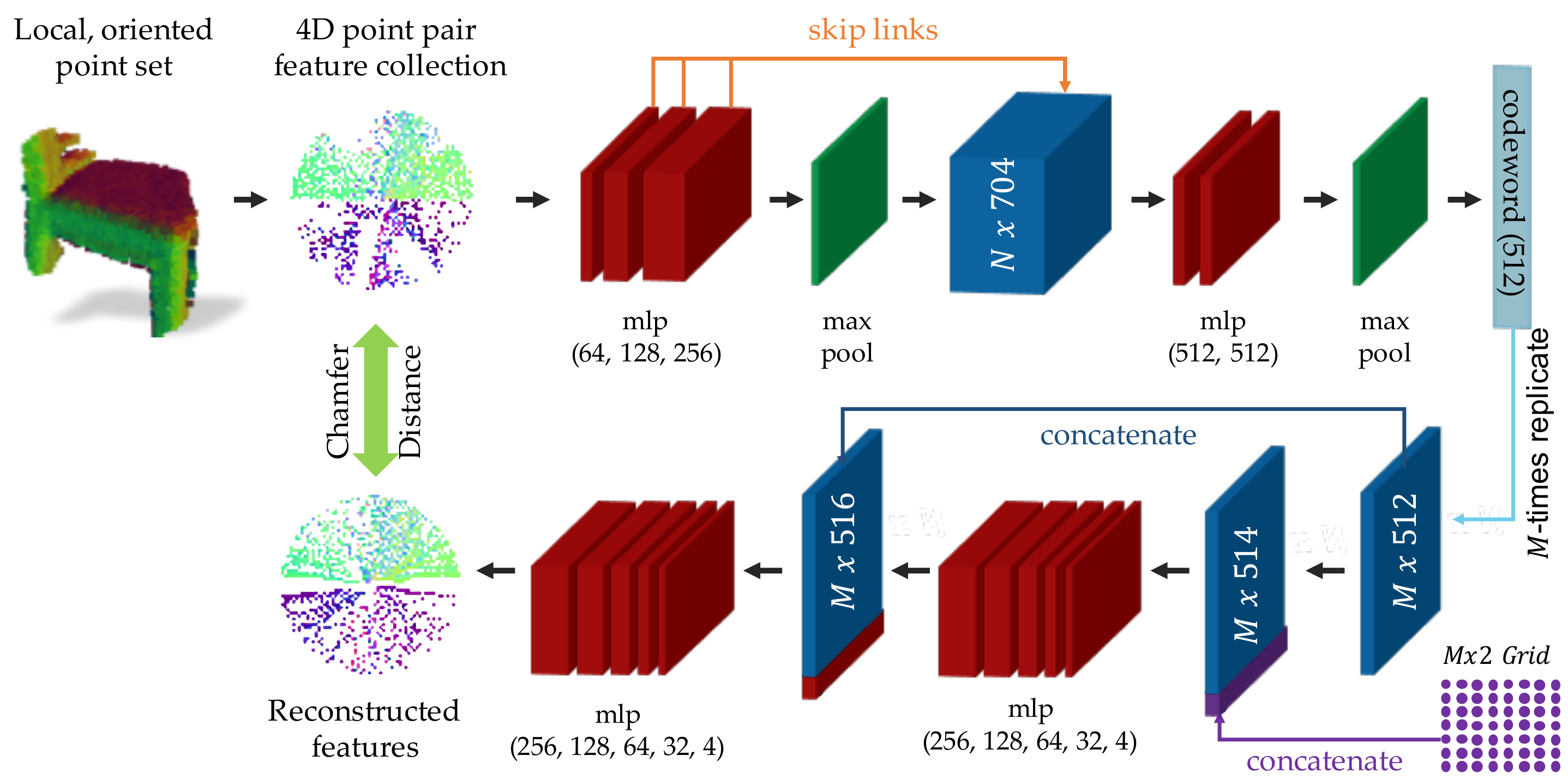}{PPF-FoldNet: The point pair feature folding network. The point cloud local patches are first converted into PPF representations, and then sent into the encoder to get compressed codewords. The decoder tries to reconstruct full PPFs from these codewords by folding. This forces the codewords to keep the most critical and discriminative information. The learned codewords are proven to be robust and effective as we will show across extensive evaluations.}{fig:teaser}{t!}

Extensive evaluations demonstrate that PPF-FoldNet outperforms the state of the art across the standard benchmarks in which severe rotations are avoided. When arbitrary rotations are introduced into the input, our descriptors outperform related approaches by a large margin including even the best competitor, Khoury et al.'s CGF~\cite{Khoury_2017_ICCV}. Moreover, we report better performance as the input sparsifies, as well as good generalization properties. Our qualitative evaluations will uncover how our network operates and give valuable interpretations. In a nutshell, our contributions can be summarized as:
\begin{itemize}
\item An auto-encoder, that unifies a PointNet encoding with a FoldingNet decoder,
\item Use of well established $4$D PPFs in this modified auto-encoder to learn rotation invariant $3$D local features without supervision.
\item A novel look at the invariance of point pair features and derived from it, a new way of visualizing PPFs and monitoring the network progress.
\end{itemize}

%% file: related.tex
\section{Prior Art}
Following their hand-crafted counterparts~\cite{shot,fpfh,usc,spin,rops}, $3$D deep learning methods started to enjoy a deep-rooted history. Initial attempts to learn from $3$D data used the naive dense voxel grid representation~\cite{zeng20163dmatch,wu20153d,maturana2015voxnet,hackel2017isprs}. While being straightforward extensions of $2$D architectures, such networks did not perform as efficiently and robustly as $2$D CNNs~\cite{imagenet}. Hence, they are superseded by networks taking into account the spatial sparsity by replacing the dense grids with octrees~\cite{octnet,octgennet,ocnn} or kd-trees~\cite{kdtreenet}. 

Another family of works acknowledges that $3$D surfaces live on $2$D submanifolds and seek to learn projections rather than the space of actual input. A reduction of dimension to two makes it possible to benefit from developments in $2$D CNNs such as Res-Nets~\cite{resnet}: LORAX~\cite{Elbaz_2017_CVPR} proposes a \textit{super-point} to depth map projection. Kehl et al.~\cite{kehl2016deep} operate on the RGB-D patches that are natural projections onto the camera plane. Huang et al.~\cite{huangKCCKY17} anchor three local cameras to each $3$D keypoint and collect multi-channel projections to learn a semi-global representation. Cao et al.~\cite{cao20173d} use spherical projections to aid object classification. Tatarchenko et al. propose convolutions in the tangent space as a way of operating on the local $2$D projection~\cite{Tat18}.
 
Point clouds can be treated as graphs by associating edges among neighbors. This paves the way to the appliance of graph convolutional networks~\cite{manessi2017dynamic}. FoldingNet~\cite{foldingnet} employs graph-based encoding layers. Wang et al.~\cite{graphnetwork} tackle the segmentation tasks on point sets via graph convolutions networks (GCNs), while Qi et al.~\cite{Qi_2017_ICCV} apply GCNs to RGB-D semantic segmentation. While showing a promising direction, the current efforts involving graphs on $3$D tasks are still supervised, try to imitate CNNs and cannot really outperform their unstructured point-processing counterparts.

Despite all developments in $3$D deep learning, there are only a handful of methods that explicitly learn generic local descriptors on $3$D data. One of the first methods that learns $3$D feature matching, also known as correspondence, is $3$DMatch~\cite{zeng20163dmatch}. It uses dense voxel grids to summarize the local geometry and learning is performed via contrastive loss. $3$DMatch is weakly supervised by task, does not learn generic descriptors, and is not invariant to rotations. PointNet~\cite{qi2016pointnet} and PointNet++~\cite{qi2017pointnet++} work directly on the unstructured point clouds and minimize a multi-task loss, resulting in local and global features. Similar to~\cite{zeng20163dmatch}, invariance is not of concern and weak supervision is essential. CGF~\cite{Khoury_2017_ICCV} combines a hand-crafted input preparation with a deep dimensionality-reduction and still uses supervision. However, the input features are not learned but only the embedding. PPFNet~\cite{ppfnet} improves over all these methods by incorporating global context, but still fails to achieve full invariance and expects supervision.

\subsection{Background}
From all of the aforementioned developments, we will now pay particular attention to three: PointNet, FoldingNet and PPFNet which combined, give our network its name. 

\paragraph{PointNet~\cite{qi2016pointnet}} Direct consumption of unstructured point input in the form of a set within deep networks began by PointNet. Qi et al. proposed to use a point-wise multi layer perceptron (MLP) and aggregated individual feature maps into a global feature by a permutation-invariant max pooling. Irrespective of the input ordering, PointNet can generate per-point local descriptors as well as a global one, which can be combined to solve different problems such as keypoint extraction, $3$D segmentation or classification. While not being the most powerful network, it clearly sets out a successful architecture giving rise to many successive studies~\cite{qi2017pointnet++,qi2017frustum,achlioptas2017,shen2017neighbors}.

\paragraph{FoldingNet~\cite{foldingnet}} While PointNet can work with point clouds, it is still a supervised architecture, and constructing unsupervised extensions like an auto-encoder on points is non-trivial as the upsampling step is required to interpolate sets~\cite{yu2018pu,qi2017pointnet++}. Yang et al. offer a different perspective and instead of resorting to costly voxelizations~\cite{wu2016learning}, propose \textit{folding}, as a strong decoder alternative. Folding warps an underlying low-dimensional grid towards a desired set, specifically a $3$D point cloud. Compared to other unsupervised methods, including GANs~\cite{wu2016learning}, FoldingNet achieves superior performance in common tasks such as classification and therefore, in PPF-FoldNet we benefit from its decoder structure, though in a slightly altered form.

\paragraph{PPFNet~\cite{ppfnet}} proposes to learn local features informed by the global context of the scene. To do so, an $N$-tuple loss is designed, seeking to find correspondences jointly between all patches of two fragments. Features learned in this way are shown to be superior than prior methods and PPFNet is reported to be the state-of-the-art local feature descriptor. However, even if Deng et al. stress the importance of learning permutation and rotation invariant features, the authors only manage to improve the resilience to Euclidean isometries slightly by concatenating PPF to the point set. Moreover, the proposed N-tuple loss still requires supervision. Our work improves on both of these aspects: It is capable of using PPFs only and operating without supervision.

%% file: method.tex
\section{PPF-FoldNet}
PPF-FoldNet is based on the idea of auto-encoding a rotation invariant but powerful representation of the point set (PPFs), such that the learned low dimensional embedding can be truly invariant. This is different to training the network with many possible rotations of the same input and forcing the output to be a canonical reconstruction. The latter would both be approximate and much harder to learn. Input to our network are local patches which, unlike PPFNet, are individually auto-encoded. The latent low dimensional vector of the auto-encoder, \textit{codeword}, is used as the local descriptor attributed to the point around which the patch is extracted. 
\subsection{Local Patch Representation} 
\label{sec:ppf}
Our input point cloud is a set of oriented points $\textbf{X} = \{\vec{x}_i \in \mathbb{R}^6\}$, meaning that each point is decorated with a local normal (e.g. tangent space) $\vec{n}\in \mathbb{R}^3$: $\vec{x}=\{\vec{p}, \vec{n}\} \in \mathbb{R}^6$. A local patch is a subset of the input $\vec{\Omega}_{\vec{x}_r} \subset \textbf{X}$ center around a reference point $\vec{x}_r$. We then encode this patch as a collection of pair features, computed between a central reference and all the other points:
\begin{equation}
\vec{F}_{\vec{\Omega}} = \{\,\vec{f}(\vec{x}_r, \vec{x}_1)  \cdots  \vec{f}(\vec{x}_r, \vec{x}_i)  \cdots \vec{f}(\vec{x}_r, \vec{x}_{N})\, \} \in \mathbb{R}^{4 \times N-1},\, i \neq r
\end{equation} 
The features between any pair (point pair features) are then defined to be a map $\vec{f}: \mathbb{R}^{12} \rightarrow \mathbb{R}^4$ sending two oriented points to three angles and the pair distance:
\begin{align}
\vec{f}: (\vec{x}_r^T,\vec{x}_i^T)^T \rightarrow (\angle(\vec{n}_r,\vec{d}), \angle(\vec{n}_i,\vec{d}), \angle(\vec{n}_r,\vec{n}_i), \Vert \vec{d} \rVert_2)^T
\end{align}
$\vec{d}=\vec{p}_r-\vec{p}_i$. An angle computation for non-normalized vectors is given in~\cite{birdal3dv2015}. Such encoding of the local geometry resembles that of PPFNet~\cite{ppfnet}, but differs in the fact that we ignore the points and normals as they are dependent on the orientation and local reference frame. We instead use pure point pair features, thereby avoiding a canonical frame computation. Note that the dimensionality of this feature is still irreducible without data loss. 
\begin{proposition}
\label{thm:ppf}
PPF representation $\vec{f}$ around $\vec{x}_r$ explains the original oriented point pair up to a rotation and reflection about the normal of the reference point. 
\end{proposition}
\begin{proof}
Let us consider two oriented points $\vec{x}_1$ and $\vec{x}_2$. We can always write the components of the associated point pair feature $\vec{f}(\vec{x}_1,\vec{x}_2)$ as follows:
\begin{equation}
\vec{n}_1^T\vec{n}_2 = f_1 \,\qquad\,
\vec{n}_1^T\vec{d}_n = f_2 \,\qquad\,
\vec{n}_2^T\vec{d}_n = f_3
\end{equation}
where $\vec{d}_n=\vec{d}/\|\vec{d}\|$. We now try to recover the original pair given its features. First, it is possible to write:
\begin{equation}
\label{eq:AAT}
\begin{bmatrix} \quad\vec{n}^T_1\quad\, \\ \quad\vec{n}^T_2\quad\, \\ \quad\vec{d}^T_n\quad\,
\end{bmatrix}
\begin{bmatrix} \vec{n}_1 & \quad\vec{n}_2 & \quad\vec{d}_n
\end{bmatrix}=
\begin{bmatrix} 1 & f_1 & f_2 \\ f_1 & 1 & f_3 \\ f_2 & f_3 & 1
\end{bmatrix}
\end{equation}
given that all vectors are of unit length. In matrix notation, Eq.~\ref{eq:AAT} can be written as $\vec{A}^T\vec{A} = \vec{K}$. Then, by singular value decomposition, $\vec{K}=\vec{U}\vec{S}\vec{V}^T$ and thus $\vec{A}=\vec{U}{\vec{S}^{1/2}}\vec{V}^T$. Note that, any orthogonal matrix (rotation and reflection) $\vec{R}$ can now be applied to $\vec{A}$ without changing the outcome: $(\vec{RA})^T{\vec{RA}}=\vec{A}^T\vec{R}^T\vec{R}\vec{A}=\vec{A}^T\vec{A}=\vec{K}$. Hence, such decomposition is up to finite-dimensional linear isometries: rotations and reflections.
Since we know that the local patch is centered at the reference point $\vec{p}_r=\vec{0}$, we are free to choose an $\vec{R}$ such that the normal vector of $\vec{p}_r$ ($\vec{n}_r$) is aligned along one of the canonical axes, say $+\vec{z}=[0,0,1]^T$ (freely chosen): 
\begin{align}
\vec{R} = \vec{I} + [\vec{v}]_x+[\vec{v}_x]^2\frac{1-n_r^z}{\|v\|}
\end{align}
where $\vec{v} = \vec{n}_r \times \vec{z}$, $n_r^z$ is the $z $ component of $\vec{n}_r$ and $\vec{I}$ is identity. $[\cdot]_x$ denotes skew symmetric cross product matrix. Because now $\vec{R}\vec{n}_r = \vec{z}$, any rotation $\theta$ and reflection $\phi$ about $\vec{z}$ would result in the same vector $\vec{z} = \vec{R}_z(\theta,\phi)\vec{z}, \, \forall \theta,\phi \in \mathbb{R}$. Any paired point can then be found in the canonical frame, uniquely up to two parameters as $\vec{p}_r \gets \|\vec{d}\| \vec{R}_z(\theta.\phi)\vec{R}\vec{d}_n$, $\vec{n}_r \gets \vec{R}_z(\theta,\phi)\vec{R}\vec{n}_r$. \qed
\end{proof}
In the case where reflections are ignored (as they are unlikely to happen in a $3$D world), this leaves a single degree of freedom, rotation angle around the normal. Also note once again that for the given local representation, the reference point $\vec{p}_r$ is common to all the point pairs.

\paragraph{Visualizing PPFs} 
PPFs exist in a $4$D space and thus it is not trivial to visualize them. While simple solutions such as PCA would work, we prefer a more geometrically meaningful and simpler solution.
Proposition \ref{thm:ppf} allows us to compute a signature of a set of point pairs by orienting the vectors ($\vec{n}_1, \vec{n}_2, \vec{d}$) individually for all points in order to align the difference vectors $\{\vec{d}_i\}$ with the $x-z$ plane by choosing an appropriate $\vec{R}_z(\theta.\phi)$. Such a transformation would not alter the features as shown. In this way, the paired points can be transformed onto a common plane (image), where the location is determined by the difference vector, in polar coordinates. The normal of the second point would not lie in this plane but can be encoded as colors in that image. Hence, it is possible to obtain a $2$D visualization, without any data loss, i.e. all components of the vector contribute to the visualization. In Fig.~\ref{fig:ppf_vis} we provide a variety of local patch and PPF visualizations from the datasets of concern.
\insertimageStar{1}{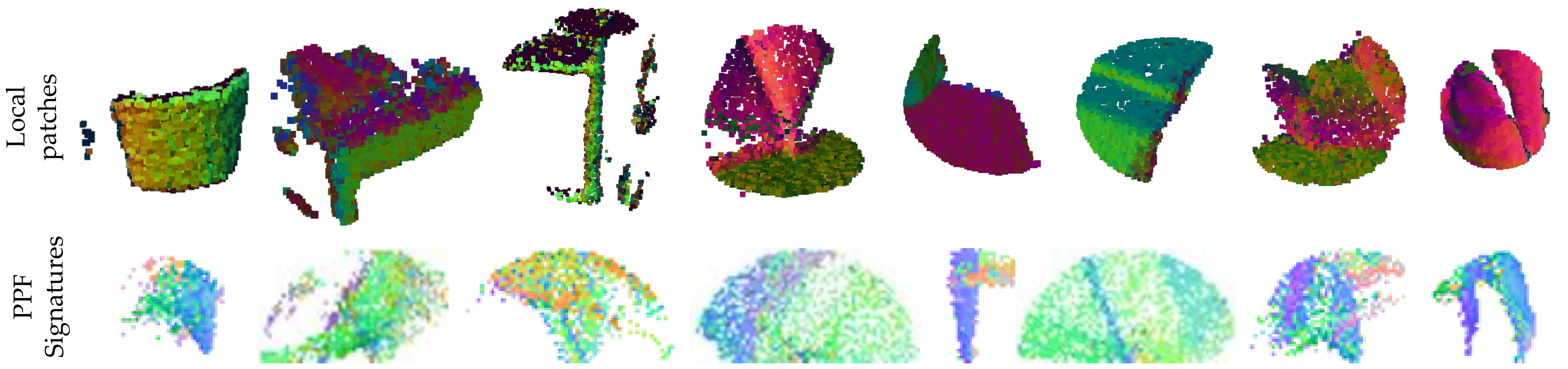}{Visualisation of some local patches and their correspondent PPF Signatures.}{fig:ppf_vis}{t!}

\subsection{PPF Auto-Encoder and Folding} 
PPF-FoldNet employs a PointNet-like encoder with skip-links and a FoldingNet-like decoding scheme. It is designed to operate on $4$D-PPFs, as summarized in Fig.~\ref{fig:teaser}.

\paragraph{Encoder} The input to our network, and thus to the encoder, is $\vec{F}_\vec{\Omega}$, a local PPF representation, as in §\ref{sec:ppf}. A three-layer, point-wise MLP (Multi Layer Perceptron) follows the input layer and subsequently a max-pooling is performed to aggregate the individual features into a global one, similar to PointNet~\cite{qi2016pointnet}. The low level features are then concatenated with this global feature using skip-links. This results in a more powerful representation. Another two-layer MLP finally redirects these features to a final encoding, the codeword, which is of dimension 512.
\begin{proposition}
\label{thm:permutation}
The encoder structure of PPF-FoldNet is permutation invariant. 
\end{proposition}
\renewcommand*{\proofname}{Sketch of the proof}
\begin{proof}
The encoder is composed of per-data-point functions (MLP), RELU layers and max-pooling, all of which either do not affect the point order or are individually shown to be permutation invariant~\cite{qi2016pointnet,foldingnet}. Moreover, it is shown that composition of functions is also invariant~\cite{foldingnet} and so is our encoder. We refer the reader to the references for further details. \qed
\end{proof}
In summary, altering the order of the PPF set will not affect the learned representation.

\paragraph{Decoder} Our decoder tries to reconstruct the whole set of point PPFs using a single codeword, which in return, also forces the codeword to be informative and distill the most distinctive information from the high-dimensional input space. However, inspired by FoldingNet, instead of trying to upsample or interpolate point sets, the decoder will try to deform a low-dimensional grid structure guided by the codeword. Each grid point is concatenated to a replica of the codeword, resulting in an $M \times 514$ vector as input to what is referred as \textit{folding operation}~\cite{foldingnet}. Folding can be a highly non-linear operation and is thus performed by two consecutive MLPs: the first folding results in a deformed grid, which is appended once again to the codewords and propagates through the second MLP, reconstructing the input PPFs. Moreover, in contrast to FoldingNet~\cite{foldingnet}, we try to reconstruct a higher dimensional set, $4$D vs $3$D ($2$D manifold); we are better off using a deeper MLP - $5$-layer as opposed to the $3$-layer of~\cite{foldingnet}.

Other than simplifying and strengthening the decoding, the folding is also beneficial in making the network interpretable. For instance, it is possible to monitor the grid during subsequent iterations and envisage how the network evolves. To do so, §\ref{sec:qualeval} will trace the PPF sets by visualizing them as described in §\ref{sec:ppf}. 

\paragraph{Chamfer Loss} Note that as size of the grid $M$, is not necessarily the same as the size of the input $N$, and the correspondences in $4$D PPF space are lost when it comes to evaluating the loss. This requires a distance computation between two unequal cardinality point pair feature sets, which we measure via the well known Chamfer metric:
\begin{equation}
d(\vec{F},\vec{\hat{F}}) = \text{max}\Bigg\{ \frac{1}{|\vec{F}|}\sum\limits_{\vec{f} \in \vec{F}} \min_{\vec{\hat{f}}\in \vec{\hat{F}}} \| \vec{f}-\vec{\hat{f}} \|_2, \,\, \frac{1}{|\vec{\hat{F}}|}\sum\limits_{\vec{f} \in \vec{\hat{F}}} \min_{\vec{f}\in\vec{F}} \| \vec{f}-\vec{\hat{f}} \|_2 \Bigg\}
\end{equation}
where $\,\hat{}\,$ operator refers to the reconstructed (estimated) set. 

\paragraph{Implementation details}
PPF-FoldNet uses Tensorflow framework \cite{abadi2016tensorflow}. The initial values of all variables are initialized randomly by Xavier's algorithm. Global loss is minimized with an ADAM optimizer~\cite{kingma2014adam}. Learning rate starts at $0.001$ and exponentially decays after every $10$ epochs, truncated at $0.0001$. We use batches of size $32$. 

%% file: results.tex
\section{Experimental Evaluation}
\subsection{Datasets and Preprocessing} To fully drive the network towards learning varieties of local $3$D geometries and gain robustness to different noises present in real data, we use the $3$DMatch Benchmark Dataset~\cite{zeng20163dmatch}. This dataset is a large ensemble of the existing ones such as Analysis-by-Synthesis \cite{valentin2016learning}, $7$-Scenes \cite{shotton2013scene}, SUN$3$D \cite{xiao2013sun3d}, RGB-D Scenes v.2 \cite{lai2014unsupervised} and Halber and Funkhouser \cite{halber2016structured}. It contains $62$ scenes in total, and we reserve $54$ of them for training and validation. $8$ are for benchmarking. $3$DMatch already provided fragments fused from $50$ consecutive depth frames of the $8$ test scenes, and we follow the same pipeline to generate fragments from the training scenes. Test fragments lack the color information and therefore we resort to using only the $3$D shape. This also makes our network insensitive to illumination changes.

Prior to operation, we downsample the fused fragments with spatial uniformity~\cite{birdal2017sampling} and compute surface normals using~\cite{Hoppe1992} in a $17$-point neighborhood. A reference point and its neighbors within $30$ cm vicinity form a local patch. The number of points in a local patch is thus flexible, which makes it difficult to organize data into regular batches. To facilitate training as well as to increase the representation robustness to noises and different point densities, each local patch is down-sampled. For a fair comparison with other methods in the literature, we use $2048$ points, but also provide an extended version that uses 5K since we are not memory bound, as for example, PPFNet~\cite{ppfnet} is. The preparation stage ends with the PPFs calculated for the assembled local patches.

\input{3dmatch_benchmark.tex}
\subsection{Accuracy Assessment Techniques} 

Let us assume that a pair of fragments $\vec{P} = \{\vec{p}_i \in \mathbb{R}^3\}$ and $\vec{Q} = \{\vec{q}_i\in \mathbb{R}^3\}$ are aligned by an associated rigid transformation $\vec{T}\in SE(3)$, resulting in a certain overlap. We then define a non-linear feature function $g(\cdot)$ for mapping from input points to feature space, and in our case, this summarizes the PPF computation and encoding as a codeword. The feature for point $\vec{p}_i$ is $g(\vec{p}_i)$, and $g(\vec{P})$ is the pool of features extracted for the points in $\vec{P}$. To estimate the rigid transformation between $\vec{P}$ and $\vec{Q}$, the typical approach finds a set of matching pairs in each fragment and associates the correspondences. The inter point pair set $\vec{M}$ is formed by the pairs $(\vec{p},\vec{q})$ that lie mutually close in the feature space by applying nearest neighbor search $NN$:
\begin{align}
\vec{M} = \{\{\vec{p}_i, \vec{q}_i\},\, &g(\vec{p}_i) = NN(g(\vec{q}_i),g(\vec{P})), \,g(\vec{q}_i) = NN(g(\vec{p}_i),g(\vec{Q}))\,\}
\end{align}
True matches set $\vec{M}_{gnd}$ is the set of point pairs with a Euclidean distance below a threshold $\tau_1$ under ground-truth transformation $\vec{T}$ .
\begin{equation}
\vec{M}_{gnd} = \{\{\vec{p}_i, \vec{q}_i\}:(\vec{p}_i, \vec{q}_i) \in \vec{M}, ||\vec{p}_i - \vec{T}\vec{q}_i||_2 < \tau_1 \}
\end{equation}
We now define an inlier ratio for $\vec{M}$ as the percentage of true matches in $\vec{M}$ as $r_{in} = {|\vec{M}_{gnd}|}/{|\vec{M}|}$.
To successfully estimate the rigid transformation based on $\vec{M}$ via registration algorithms, $r_{in}$ needs to be greater than $\tau_2$. For example, in a common RANSAC pipeline, achieving $99.9\%$ confidence in the task of finding a subset with at least $3$ correct matches $\vec{M}$, with an inlier ratio $\tau_2 = 5\%$ requires at least $55258$ iterations. Theoretically, given  $r_{in} > \tau_2$, it is highly probable a reliable local registration algorithm would work, regardless of the robustifier. Therefore instead of using the local registration results to judge the quality of features, which would be both slow and not very straightforward, we define $\vec{M}$ with $r_{in} > \tau_2$ votes for a correct match of two fragments.

Each scene in the benchmark contains a set of fragments. Fragment pairs $\vec{P}$ and $\vec{Q}$ having an overlap above $30\%$ under the ground-truth alignment are considered to match.  Together they form the set of fragment pairs $\vec{S} = \{(\vec{P}, \vec{Q})\} $ that are used in evaluations. The quality of features is measured by the recall $R$ of fragment pairs matched in $\vec{S}$:

\begin{equation}
R = \frac{1}{|\vec{S}|} \sum\limits_{i=1}^{|\vec{S}|} \mathbbm{1} \Bigg( r_{in}\big(\vec{S}_i=(\vec{P}_i,\vec{Q}_i)\big) > \tau_2 \Bigg)
\end{equation}
\input{3dmatch_benchmark_rotation.tex}
\subsection{Results} 
\paragraph{Feature quality evaluation}
We first compare the performance of our features against the well-accepted works on the $3$DMatch benchmark with $\tau_1 = 10$ cm and $\tau_2=5\%$. Tab.~~\ref{tab:3dmatchbenchmark} tabulates the findings. The methods selected for comparison comprise $3$ handcrafted features (Spin Images~\cite{spin}, SHOT~\cite{shot}, FPFH~\cite{fpfh}) and $4$ state-of-the-art deep learning based methods ($3$DMatch~\cite{zeng20163dmatch}, CGF~\cite{Khoury_2017_ICCV}, PPFNet~\cite{ppfnet}, FoldingNet~\cite{foldingnet}). Note that FoldingNet has never been tested on local descriptor extraction before. It is apparent that, overall, our PPF-FoldNet could match far more fragment pairs in comparison to the other methods, except for scenes \textit{Kitchen} and \textit{Home}, where PPFNet and $3$DMatch achieve a higher recall respectively. In all the other cases, PPF-FoldNet outperforms the state of the art by a large margin, $>9\%$ on average. PPF-FoldNet has a recall of $68.04\%$ when using $2$K sample points (the same as PPFNet), while PPFNet remains on $62.32\%$. Moreover, because PPF-FoldNet has no memory bottleneck, it can achieve an additonal $3\%$ improvement in comparison with the $2$K version, when 5K points are used. Interestingly, FPFH is also constructed from a type of PPF features~\cite{fpfh}, but in a form of manual histogram summarization. Compared to FPFH, PPF-FoldNet has $32.15\%$  and $35.93\%$ higher recall using $2$K and $5$K points respectively. It demonstrates the unprecedented strength of our advanced method in compressing the PPFs. In order to optimally reconstruct PPFs in the decoder, the network forces the bottleneck codeword to be compact as well as distilling the most critical and distinctive information in PPFs. 

To illustrate that parameters in the evaluation metric are not tuned for our own good, we also repeat the experiments with different $\tau_1$ and $\tau_2$ values. The results are shown in Fig.~\ref{subfig:inlierthresh} and Fig.~\ref{subfig:distancethresh}. In Fig.~\ref{subfig:inlierthresh}, $\tau_1$ is fixed at $10$ cm, $\tau_2$ increases gradually from $1\%$ to $20\%$. When $\tau_2$ is above $4\%$, PPF-FoldNet always has a higher recall than the other methods. Below $4\%$, some other methods may obtain a higher recall but this is too strict for most of the registration algorithms anyway. It is further noteworthy that when $\tau_2$ is set to $20\%$, the point where PPF-FoldNet still gets a recall above $20\%$, the performance of the other methods falls below $5\%$. This justifies that PPF-FoldNet is capable of generating many more sets of matching points with a high inlier ratio $r_{in}$. This offers a tremendous benefit for the registration algorithms. In Fig.~\ref{subfig:distancethresh}, $\tau_2$ is fixed at $5\%$, $\tau_1$ increases gradually from $0$ cm to $20$ cm. When $\tau_1$ is smaller than $12$ cm, PPF-FoldNet consistently generates higher recall. This finding indicates that PPF-FoldNet matches more point pairs with a small distance error in the Euclidean space, which could efficiently decrease the rigid transformation estimation errors.

\input{evaluation_plots.tex}

\paragraph{Tests on rotation invariance}
To demonstrate the outstanding rotation-invariance property of PPF-FoldNet, we take random fragments out of the evaluation set and gradually rotate them around the $z$-axis from $60^{\circ}$ to $360^{\circ}$ in steps of $60^{\circ}$. The matching results are shown in Fig.~\ref{subfig:zrotation}. As expected, both PPFNet and $3$DMatch perform poorly as they operate on rotation-variant input representations. Hand crafted features or CGF also demonstrate robustness to rotations thanks to the reliance on the local reference frame (LRF). However, PPF-FoldNet stands out as the best approach with a much higher recall which furthermore does not require computation of local reference frames.

To further test how those methods perform under situations with severe rotations, we rotate all the fragments in $3$DMatch benchmark with randomly sampled axes and angles over the whole rotation space, and introduce a new benchmark -- \textit{Rotated $3$DMatch Benchmark}. The same evaluation is once again conducted on this new benchmark. Keeping the accuracy evaluations identical, our results are shown in Tab.~\ref{tab:rot3dmatchbenchmark}. $3$DMatch and PPFNet completely failed under this new benchmark because of the variables introduced by large rotations. Once again, PPF-FoldNet, surpasses all other methods, achieving the best results in all the scenes, predominates the runner-up CGF by large margins of $18.78\%$ and $23.24\%$ respectively when using $2$K and $5$K points.
\paragraph{Sparsity evaluation}
Thanks to the sparse representation of our input, PPF-FoldNet is also robust in respect of the changes in point cloud density and noise. Fig.\ref{subfig:sparsity} shows the performance of different methods when we gradually decrease the points in the fragment from $100\%$ to only $6.25\%$. We can see that PPF-FoldNet is least affected by the decrease in point cloud density. In particular, when only $6.25\%$ points are left in the fragments, the recall for PPF-FoldNet is still greater than $50\%$ while PPFNet remains around $12\%$ and the other methods almost fail. The results of PPFNet and PPF-FoldNet together demonstrate that PPF representation offers more robustness in respect of point densities, which is a common problem existing in many point cloud representations.    
\begin{table}[t!]
  \centering
  \setlength{\tabcolsep}{3pt}
  \caption{Accuracy comparison of different PPF representations.}
    \begin{tabular}{lccccccccc}
          & \multicolumn{1}{l}{Kitchen} & \multicolumn{1}{l}{Home 1} & \multicolumn{1}{l}{Home 2} & \multicolumn{1}{l}{Hotel 1} & \multicolumn{1}{l}{Hotel 2} & \multicolumn{1}{l}{Hotel 3} & \multicolumn{1}{l}{Study} & \multicolumn{1}{l}{MIT Lab} & \multicolumn{1}{l}{Average} \\
    \midrule
    PPFH  & 0.534 & 0.622 & 0.486 & 0.341 & 0.346 & 0.574 & 0.233 & 0.351 & 0.436 \\
     Bobkov1     & 0.514 & 0.635 & 0.510 & 0.403 & 0.433 & 0.611 & 0.281 & 0.481 & 0.483 \\
    Our-PPF & 0.506 & 0.635 & 0.495 & 0.350 & 0.385 & 0.667 & 0.267 & 0.403 & 0.463 \\
    \end{tabular}%
  \label{tab:ppf}%
\end{table}%
\paragraph{Can PPF-FoldNet operate with different PPF constructions?} We now study 3 identical networks, trained for 3 different PPF formulations: ours, PPFH (the PPF used in FPFH~\cite{fpfh}) and Bobkov1 et al.~\cite{bobkov2018noise}. The latter has an added component of \textit{occupancy ratio} based on grid space. We use a subset of $3$DMatch benchmark to train all networks for a fixed number of iterations and test on the rotated fragments. Tab.~\ref{tab:ppf} presents our findings: all features perform similarly. Thus, we do not claim the superiority of our PPF representation, but stress that it is simple, easy to compute, intuitive and easy to visualize. Due to the voxelization, \textit{Bobkov1} is significantly slower than the other methods, and due to the lack of an LRF, our PPF is faster than \textit{PPFH}'s. Using stronger pair primitives would favor PPF-FoldNet as our network is agnostic to the PPF construction.
\paragraph{Runtime} We run our algorithm on a machine loaded with NVIDIA TitanX Pascal GPU and an Intel Core i$7$ $3.2$GHz CPU. On this hardware, computing features of an entire fragment via FPFH~\cite{fpfh} takes $31.678$ seconds, whereas PPF-FoldNet achieves a $10\times$ speed-up with $3.969$ seconds, despite having similar theoretical complexity. In particular, our input preparation for PPF extraction runs in $2.616$ seconds, and the inference in $1.353$. This is due to 1) PPF-FoldNet requiring only a single pass over the input, 2) our efficient network accelerated on GPU powered Tensorflow.
\insertimageStar{1}{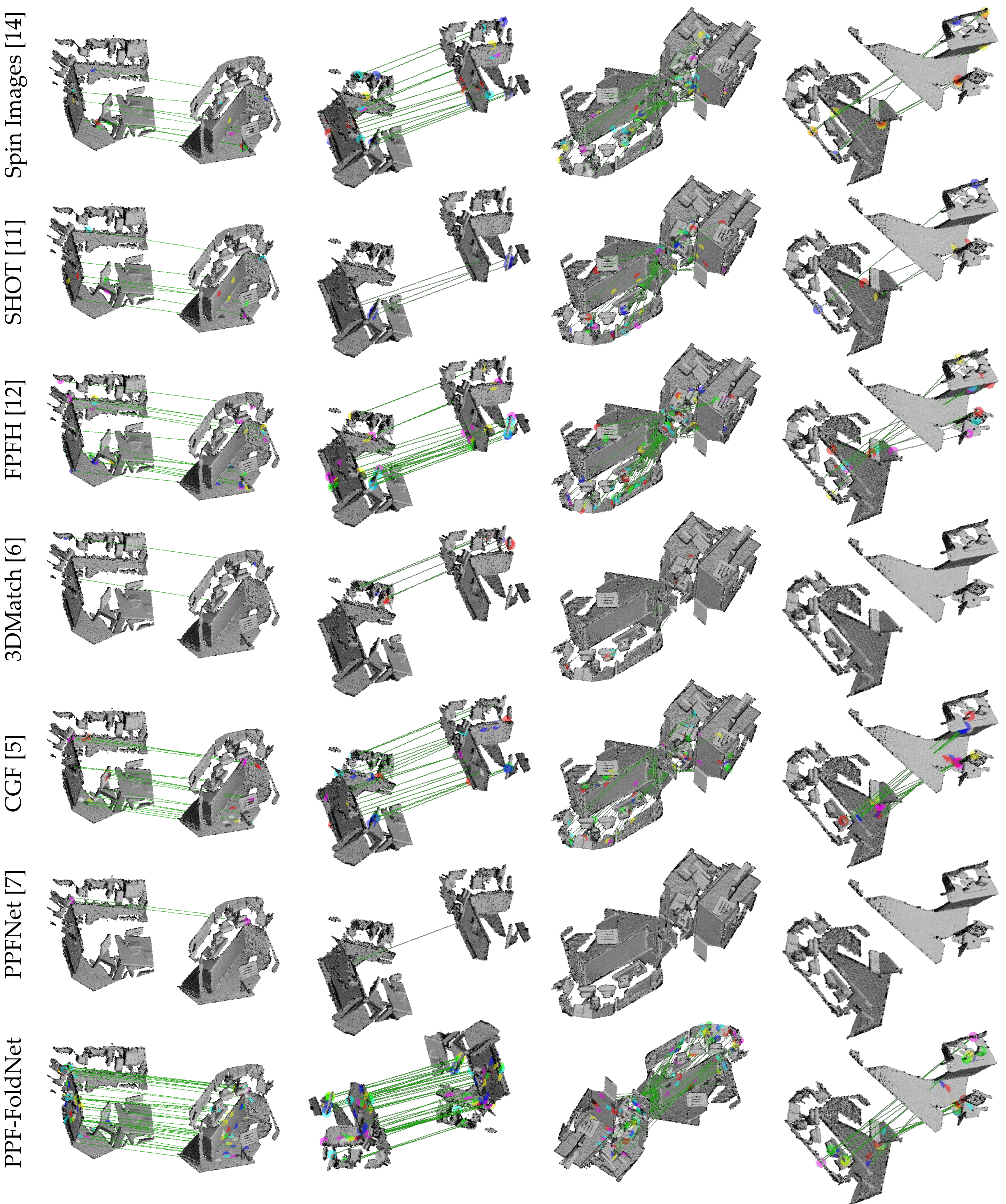}{Qualitative results of matching across different fragments and for different methods. When severe transformations are present, only hand-crafted algorithms, CGF and our method achieves satisfactory matches. However, for PPF-FoldNet, the number of matches are significantly larger.}{fig:matching}{t!}
\insertimageStar{1}{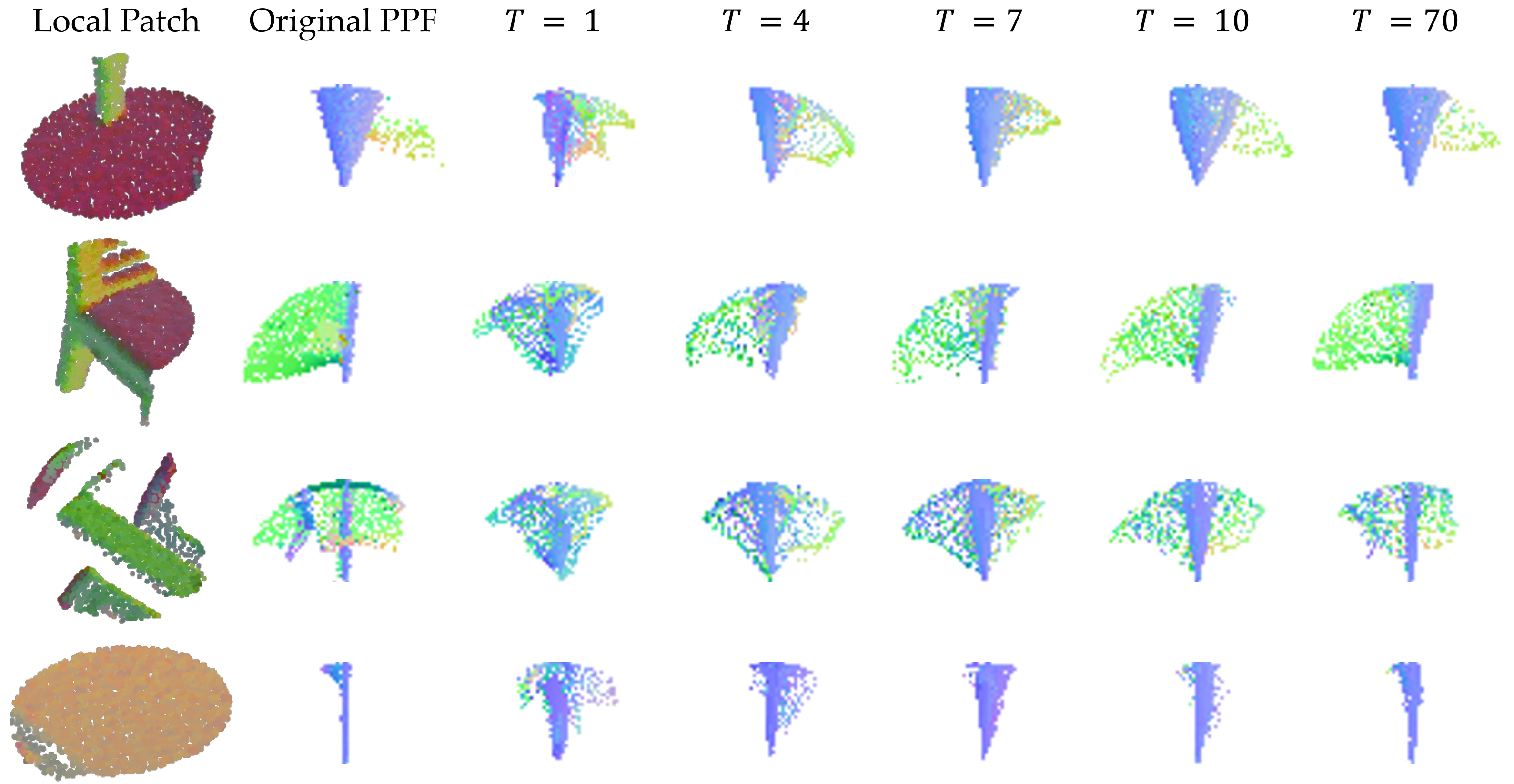}{Visualizing signatures of reconstructed PPFs. As the training converges, the reconstructed PPF signatures become closer to the original signatures. Our network reveals the underlying structure of the PPF space.}{fig:evolution}{t!}
\subsection{Qualitative Evaluations} 
\label{sec:qualeval}
\paragraph{Visualizing the matching result} From the quantitative results, PPF-FoldNet is expected to have better and more correct feature matches, especially when arbitrary rigid transformations are applied. To show this visually, we run different methods and ours across several fragments undergoing varying rotations. In Fig.~\ref{fig:matching} we show the matching regions, over uniformly sampled~\cite{birdal2017sampling} keypoints on these fragments. It is clear that our algorithm performs the best among all others in discovering the most correct correspondences.
\paragraph{Monitoring network evolution} As our network is interpretable, it is tempting to qualitatively analyze the progress of the network. To do that we record the PPF reconstruction output at discrete time steps and visualize the PPFs as explained in §~\ref{sec:ppf}. Fig~\ref{fig:evolution} shows such a visualization for different local patches. First, thanks to the representation power, our network achieves high fidelity recovery of PPFs. Note that even though the network starts from a random initialization, it can quickly recover a desired point pair feature set, even after only a small number of iterations. Next, for similar local patches (top and bottom rows), the reconstructions are similar, while for different ones, different. 

\paragraph{Visualizing the latent space}
We now attempt to visualize the learned latent space and assess whether the embedding is semantically meaningful. To do so, we compute a set of codewords and the associated PPF signatures. We then run the Barnes Hut T-SNE algorithm~\cite{tsne,barneshut} on the extracted codewords and form a two-dimensional embedding space, as shown in Fig.~\ref{fig:tsne}. At each $2$D location we paint the PPF signature and thereby illustrate the distribution of PPFs along the manifold. We also plot the original patches which generated the codewords and their corresponding signatures as cutouts. Presented in Fig.~\ref{fig:tsne}, whenever the patches are geometrically and semantically close, the computed descriptors are close, and whenever the patches have less physical similarity, they are embedded into different parts of the space. This provides insight into the good performance and meaningfulness in the relationships our network could learn.
\insertimageStar{1}{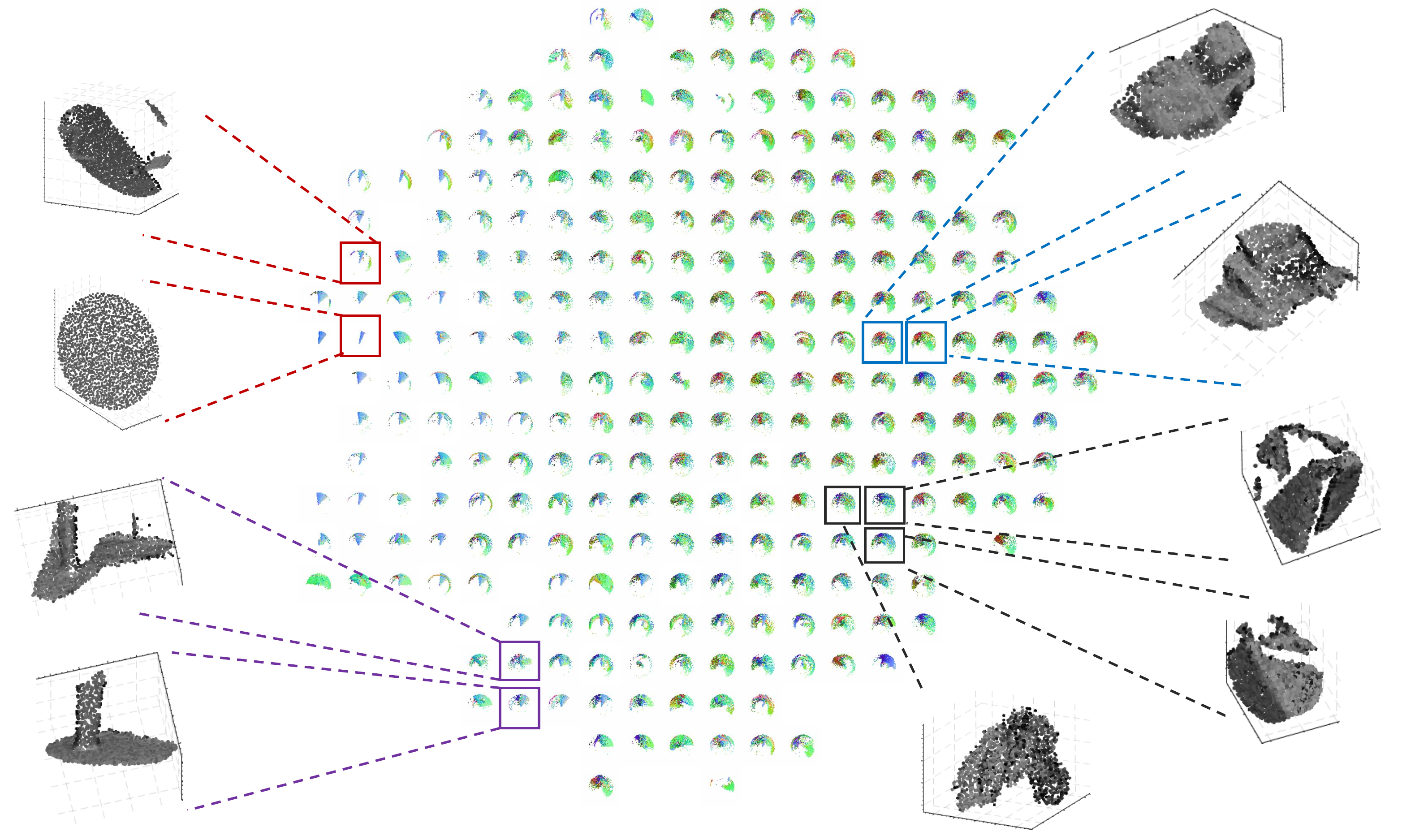}{Visualization of the latent space of codewords, associated PPFs and samples of clustered local $3$D patches using TSNE~\cite{tsne,barneshut}.}{fig:tsne}{t!}
\insertimageStar{1}{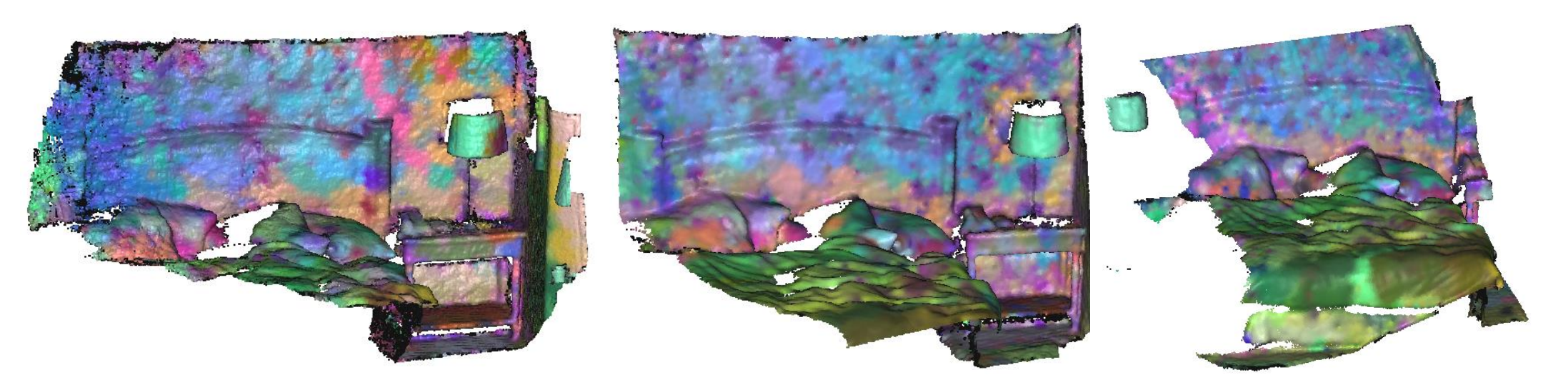}{Visualization of the latent feature space on fragments fused from different views. To map each feature to a color on the fragment, we use TSNE embedding~\cite{tsne}. We reduce the dimension to three and associate each low dimensional vector with an RGB color.}{fig:rotation}{t!}
In a further experiment, we extract a feature at each location of the point cloud. Then, we reduce the dimension of the latent space to three via TSNE~\cite{tsne}, and colorize each point by the reduced feature vector. Qualitatively justifying the repeatibility of our descriptors, the outcome is shown in Fig.~\ref{fig:rotation}. Note that, descriptors extracted by the proposed approach lead to similar colors in matching regions among the different fragments. 

%% file: 3dmatch_benchmark.tex
\begin{table*}[t!]
  \centering
  \caption{Our results on the standard 3DMatch benchmark. \textit{Red Kitchen} data is from 7-scenes~\cite{shotton2013scene} and the rest imported from SUN3D~\cite{xiao2013sun3d}.}
  \resizebox{\textwidth}{!}{
    \begin{tabular}{lcccccccccc}
    \toprule
          & Spin Image~\cite{spin} & SHOT~\cite{shot}  & FPFH~\cite{fpfh}    &
          3DMatch~\cite{zeng20163dmatch} &
          CGF~\cite{Khoury_2017_ICCV} &
          PPFNet~\cite{ppfnet} &
          FoldNet~\cite{foldingnet} & 
          Ours &
          Ours-5K \\
    \midrule
    Kitchen & 0.1937 & 0.1779 & 0.3063  & 0.5751 & 0.4605 & \textbf{0.8972} &  0.5949 & 0.7352 & 0.7866 \\
    Home 1 & 0.3974 & 0.3718 & 0.5833  & 0.7372 & 0.6154 & 0.5577 & 0.7179 & 0.7564 & \textbf{0.7628} \\
    Home 2 & 0.3654 & 0.3365 & 0.4663  & \textbf{0.7067} & 0.5625 & 0.5913 & 0.6058 & 0.6250 & 0.6154 \\
    Hotel 1 & 0.1814 & 0.2080 & 0.2611 &  0.5708 & 0.4469 & 0.5796 & 0.6549 & 0.6593 & \textbf{0.6814} \\
    Hotel 2 & 0.2019 & 0.2212 & 0.3269 & 0.4423 & 0.3846 & 0.5769 & 0.4231 & 0.6058 & \textbf{0.7115} \\
    Hotel 3 & 0.3148 & 0.3889 & 0.5000 & 0.6296 & 0.5926 & 0.6111 & 0.6111 & 0.8889 & \textbf{0.9444} \\
    Study & 0.0548 & 0.0719 & 0.1541  & 0.5616 & 0.4075 & 0.5342 & \textbf{0.7123} & 0.5753 & 0.6199 \\
    MIT Lab & 0.1039 & 0.1299 & 0.2727  & 0.5455 & 0.3506 & \textbf{0.6364} &  0.5844 & 0.5974 & 0.6234 \\
    \midrule
    Average & 0.2267 & 0.2382 & 0.3589  & 0.5961 & 0.4776 & 0.6231 & 0.6130 & 0.6804 & \textbf{0.7182} \\
    \bottomrule
    \end{tabular}%
  \label{tab:3dmatchbenchmark}%
  }
\end{table*}%

%% file: 3dmatch_benchmark_rotation.tex
\begin{table*}[t!]
  \centering
  \caption{Our results on the rotated 3DMatch benchmark. \textit{Red Kitchen} data is from 7-scenes~\cite{shotton2013scene} and the rest imported from SUN3D~\cite{xiao2013sun3d}.}
  \resizebox{\textwidth}{!}{
    \begin{tabular}{lccccccccc}
    \toprule
          & Spin Image~\cite{spin} & SHOT~\cite{shot}  & FPFH~\cite{fpfh}  & 
          3DMatch~\cite{zeng20163dmatch} &
          CGF~\cite{Khoury_2017_ICCV} &
          PPFNet~\cite{ppfnet} &
          FoldNet~\cite{foldingnet} &
          Ours &
          Ours-5K \\
    \midrule
    Kitchen & 0.1779 & 0.1779 & 0.2905 & 0.0040 & 0.4466 & 0.0020 & 0.0178& 0.7352 & \textbf{0.7885} \\
    Home 1 & 0.4487 & 0.3526 & 0.5897 & 0.0128 & 0.6667 & 0.0000 & 0.0321    & 0.7692 & \textbf{0.7821} \\
    Home 2 & 0.3413 & 0.3365 & 0.4712 & 0.0337 & 0.5288 & 0.0144 & 0.0337 & 0.6202 & \textbf{0.6442} \\
    Hotel 1 & 0.1814 & 0.2168 & 0.3009 & 0.0044 & 0.4425 & 0.0044 & 0.0133 & 0.6637 & \textbf{0.6770} \\
    Hotel 2 & 0.1731 & 0.2404 & 0.2981 & 0.0000     & 0.4423 & 0.0000 & 0.0096    & 0.6058 & \textbf{0.6923} \\
    Hotel 3 & 0.3148 & 0.3333 & 0.5185 & 0.0096 & 0.6296 & 0.0000 & 0.0370    & 0.9259 & \textbf{0.9630} \\
    Study & 0.0582 & 0.0822 & 0.1575 & 0.0000     & 0.4178 & 0.0000  & 0.0171   & 0.5616 & \textbf{0.6267} \\
    MIT Lab & 0.1169 & 0.1299 & 0.2857 & 0.0260 & 0.4156 & 0.0000 & 0.0260    & 0.6104 & \textbf{0.6753} \\
    \midrule
    Average & 0.2265 & 0.2337 & 0.3640 & 0.0113 & 0.4987 & 0.0026 & 0.0233 & 0.6865 & \textbf{0.7311} \\
    \bottomrule
    \end{tabular}%
  \label{tab:rot3dmatchbenchmark}%
  }
\end{table*}%

%% file: evaluation_plots.tex
\begin{figure*}[t!]
\centering
\subfigure[]{
\includegraphics[width=0.232505\linewidth, clip=true]{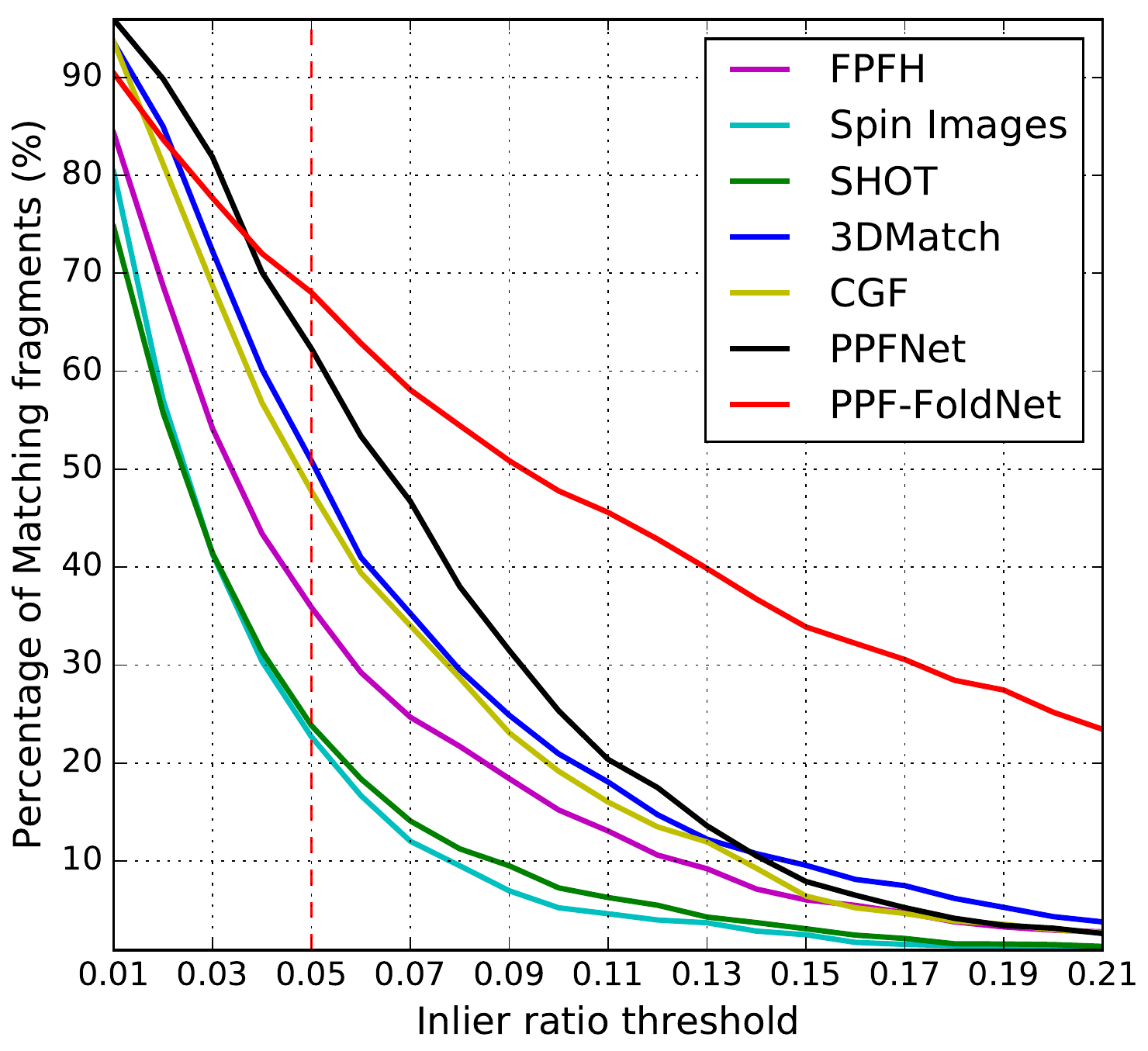}
\label{subfig:inlierthresh}}
\subfigure[]{
\includegraphics[width=0.2315\linewidth, clip=true]{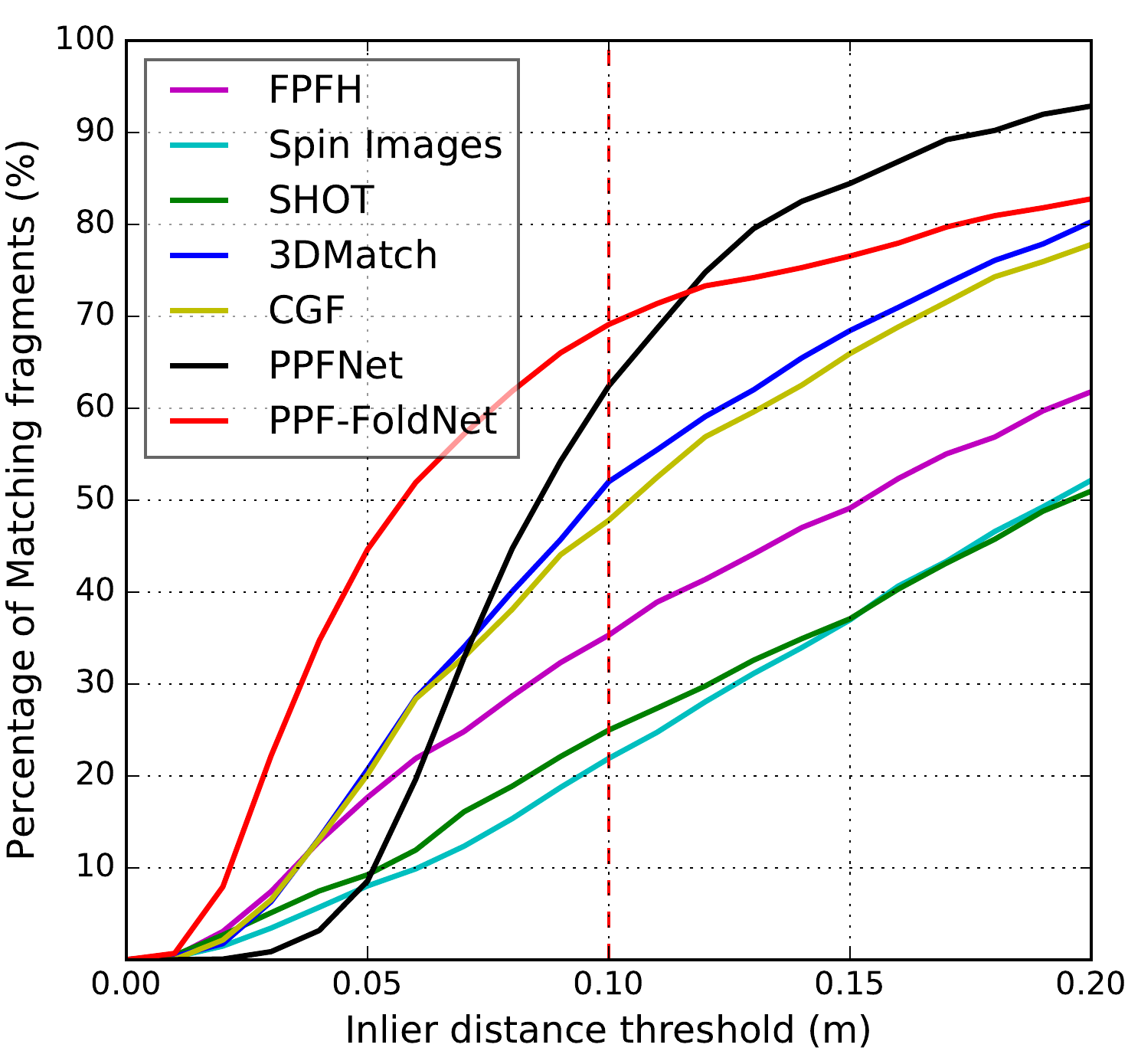}
\label{subfig:distancethresh}}
\subfigure[]{
\includegraphics[width=0.2325\linewidth, clip=true]{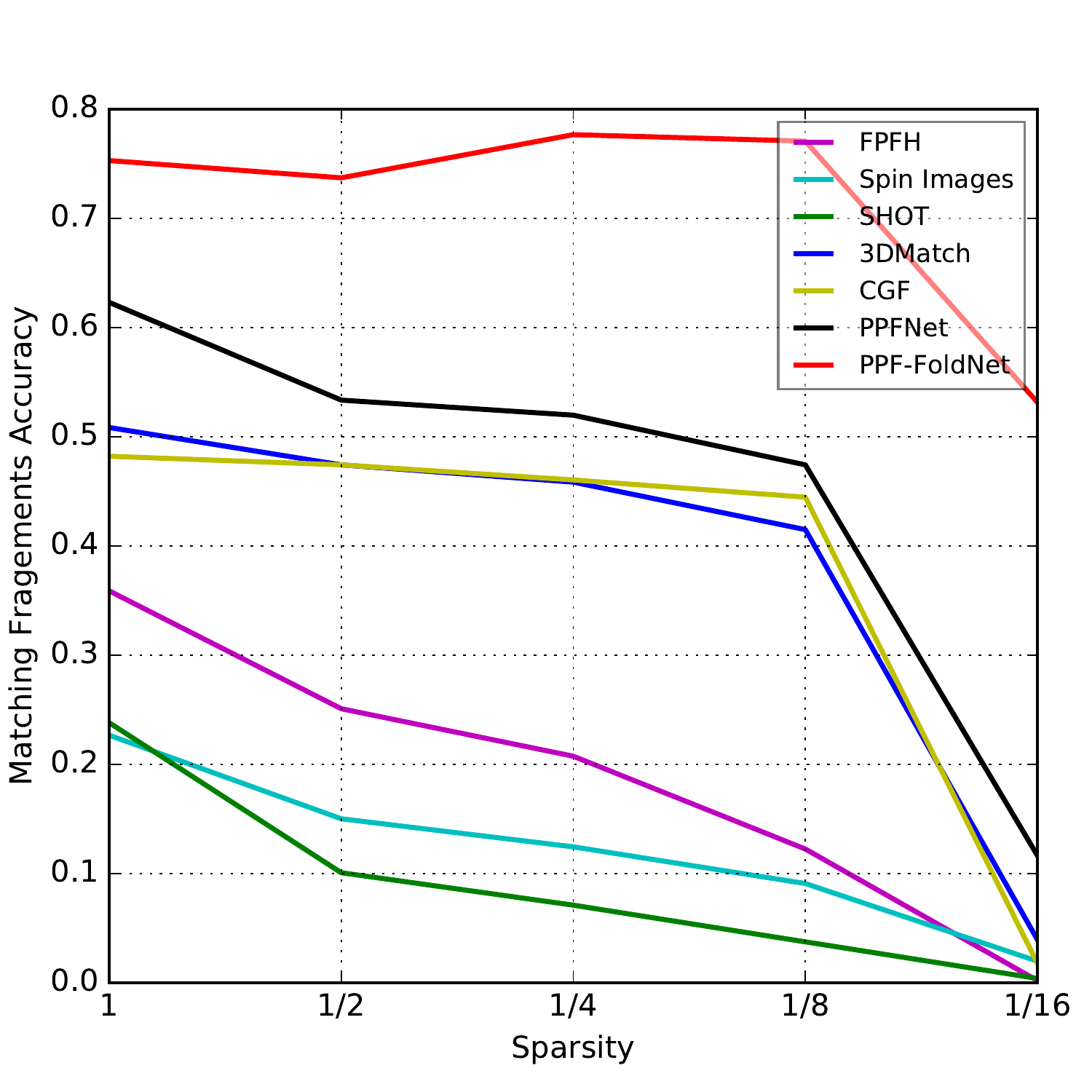}
\label{subfig:sparsity}}
\subfigure[]{
\includegraphics[width=0.23305\linewidth, clip=true]{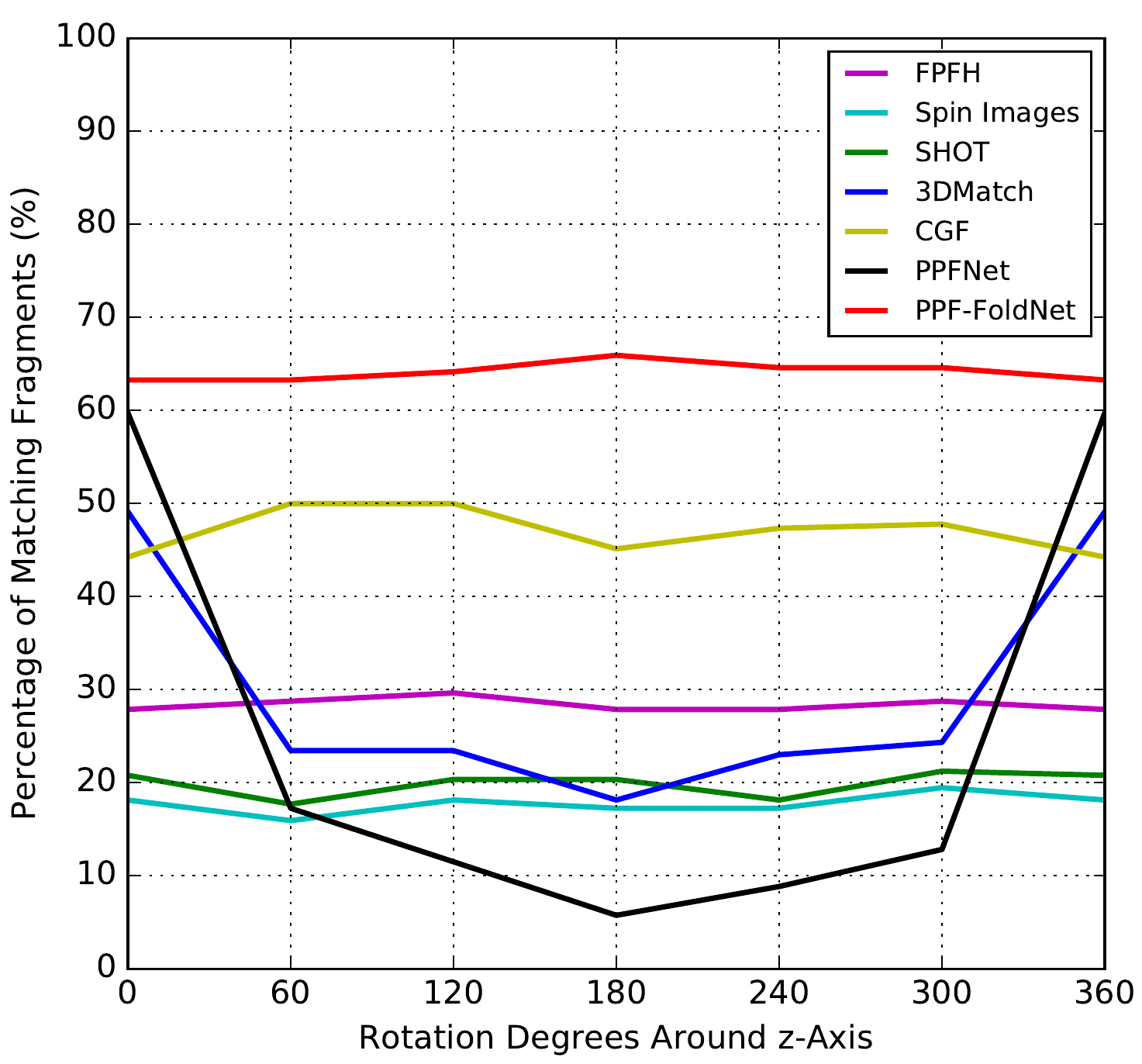}
\label{subfig:zrotation}}
\caption{Evaluations on 3DMatch benchmark: \textbf{(a)} Results of different methods under varying inlier ratio threshold \textbf{(b)} Results of different methods under varying point distance threshold \textbf{(c)} Evaluating robustness again point density
\textbf{(d)} Evaluations against rotations around $z$-axis}
\label{fig:evaluations}
\end{figure*}

%% file: conclusion.tex
\section{Concluding Remarks}
We have presented PPF-FoldNet, an unsupervised, rotation invariant, low complexity, intuitive and interpretable network in order to learn $3$D local features solely from point geometry information. Our network is built upon its contemporary ancestors, PointNet, FoldingNet $\&$ PPFNet and it inherits best attributes of all. Despite being rotation invariant, we have outperformed all the state-of-the-art descriptors, including supervised ones even in the standard benchmarks under challenging conditions with varying point density.
We believe PPF-FoldNet offers a promising new approach to the important problem of unsupervised $3$D local feature extraction and see this as an important step towards unsupervised revolution in $3$D vision.

Our architecture can be extended in many directions. One of the most promising of those would be to adapt our features towards tasks like classification and object pose estimation. We conclude with the hypothesis that the generalizability in our unsupervised network should transfer easily into solving other similar problems, giving rise to an open application domain.

%% file: appendix.tex
\subsection{Evaluations on Generalizability}
Unfortunately, state-of-the-art methods for learning 3D features rely heavily on the availability of extensively annotated data - such as ground truth matches between the pairs. In 3DMatch, CGF and PPFNet, contrastive, triplet and N-tuple losses are used respectively. Such supervision prevents these methods from immediately extending to different datasets, without fine-tuning. It might occasionally be prohibitive to even obtain labeled data for new datasets.

\begin{wrapfigure}[16]{r}{0.5\textwidth}
	\centering\vspace{-37.5pt}
			\includegraphics[width=0.5\textwidth, clip=true]{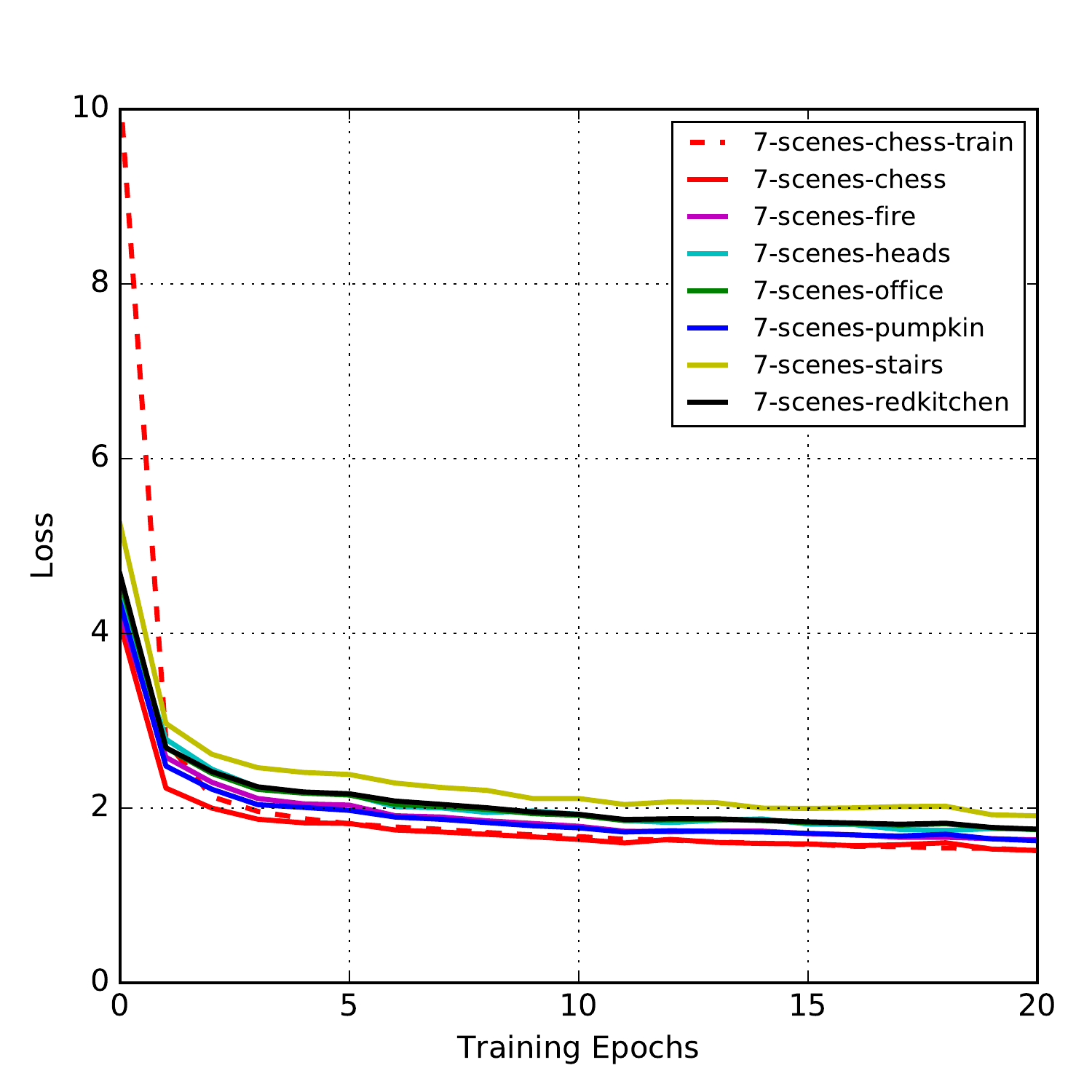}\vspace{-3mm}
        	\caption{Generalizability Test}
	        \label{fig:general}
\end{wrapfigure}
This is different for PPF-FoldNet, where we learn a completely unsupervised representation. Thanks to the novel auto-encoder, we can operate on any available dataset without requiring auxiliary label information. Therefore, we are motivated to believe that PPF-FoldNet would generalize better to unseen data, even with a small subset of unlabeled data being available. 

To test the aforementioned hypothesis, we propose an experimentation, where PPF-FoldNet is trained on a small portion of scenes and tested on multiple different ones. We divide the \textit{Chess} scene from 7-scenes dataset into train and test splits, train PPF-FoldNet from scratch, and measure the loss on the test data including \textit{Chess} and other six scenes as well. Note that the remaining six scenes do not contribute to training. For all datasets, we plot the loss curves among iterations in Fig.~\ref{fig:general}. The dashed line stands for average loss of training data for each epoch, and the other curves depict the average loss of test data from each scene respectively. As the training proceeds, the losses for all 7 datasets decrease following similar trend. Achieving such residual values validates that PPF-FoldNet can generalize to unseen input.

\subsection{Additional Visualizations of Matching}
Fig.~\ref{fig:matching2} presents further qualitative analysis involving matching of rotated fragments and across all evaluated methods.

\insertimageStar{1}{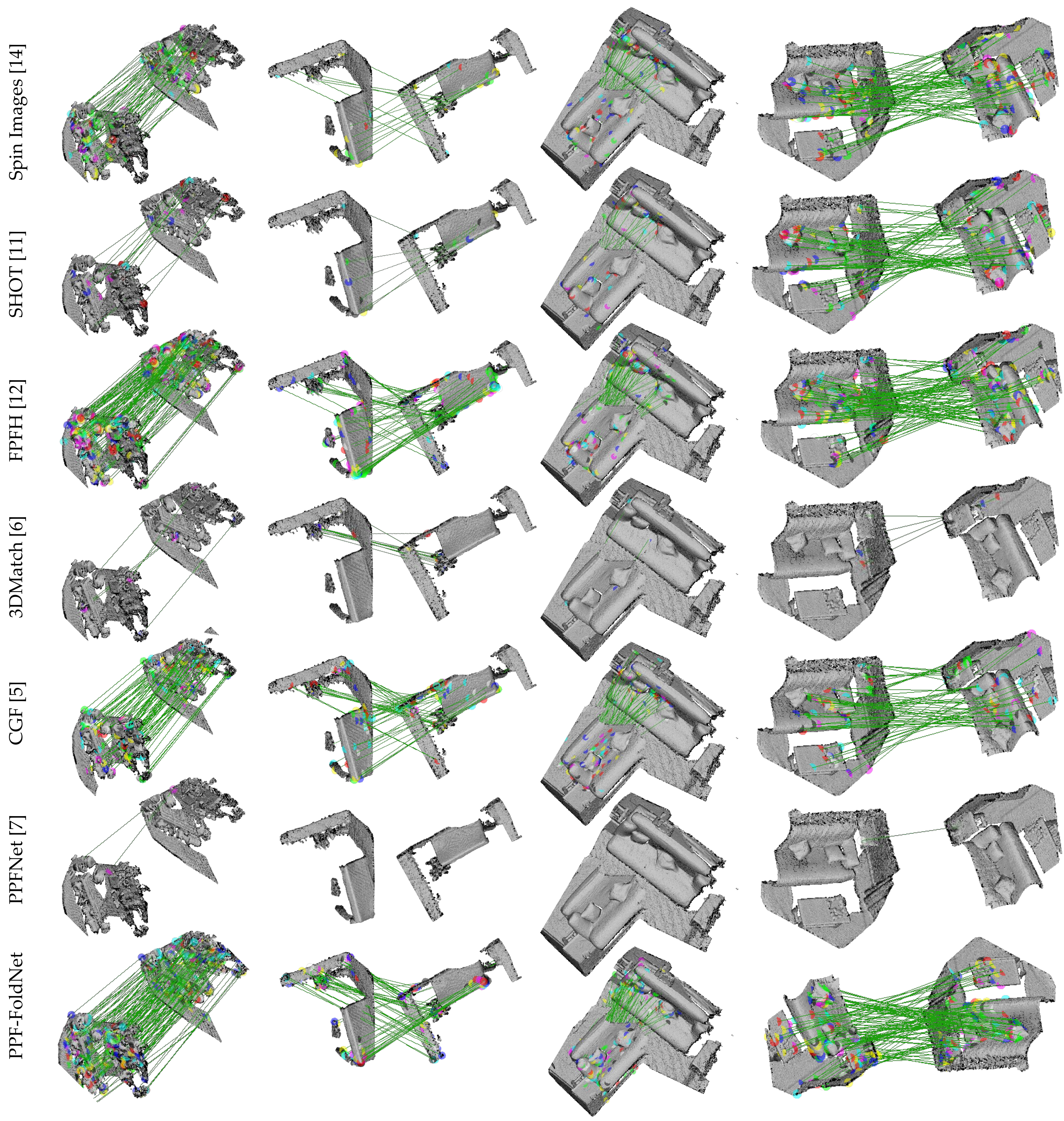}{Qualitative results of matching across different fragments and for different methods. When transformations involving rotations are present, only hand-crafted algorithms, CGF and our method achieves satisfactory matches. However, for our PPF-FoldNet, the number of matches are significantly larger.}{fig:matching2}{t!}